%% file: Template.tex
\documentclass{article}
\usepackage{spconf,amsmath,graphicx,hyperref}
\usepackage[numbers,sort&compress]{natbib}
\usepackage{amssymb}
\usepackage{amsmath}
\usepackage{amsthm}
\usepackage{graphicx}
\usepackage{booktabs}
\usepackage{multirow}
\usepackage[utf8]{inputenc}
\usepackage{rotating}
\usepackage{hyperref}  % for clickable links and emails
\usepackage{float}
\usepackage{colortbl}
\usepackage{xcolor}
\usepackage{float}
\usepackage{tcolorbox}
\definecolor{best}{RGB}{144,238,144}   % 淡绿色
\definecolor{second}{RGB}{255,218,185} % 淡橙色
\definecolor{lightpurple}{rgb}{0.862, 0.918, 0.992}

\newtheorem{theorem}{Theorem}
\newtheorem{assumption}{Assumption}
\title{HGAN-SDEs: Learning Neural Stochastic Differential Equations\\with Hermite-Guided Adversarial Training}
\name{Yuanjian Xu\textsuperscript{1}, Jianing Hao\textsuperscript{2}, Shuai Yuan\textsuperscript{3}, Guang Zhang\textsuperscript{1,*}\thanks{*Corresponding author.}}
\address{
\ninept
\textsuperscript{1}Thrust of Financial Technology, The Hong Kong University of Science and Technology (Guangzhou), Guangzhou, China \\
\ninept
\textsuperscript{2}Thrust of Data Science and Analytics, The Hong Kong University of Science and Technology (Guangzhou), Guangzhou, China \\
\ninept
\textsuperscript{3}School of Software and Microelectronics, Peking University, Beijing, China \\
\ninept
\textsuperscript{}\{yxu085, jhao768\}@connect.hkust-gz.edu.cn, yuanshuai@stu.pku.edu.cn, guangzhang@hkust-gz.edu.cn
}

\begin{document}
\ninept
\maketitle

\input{src/1_abstract}
%
\input{src/2_introduction}
\input{src/3_relatedwork}
\input{src/4_method}
\input{src/5_results}

\bibliographystyle{elsarticle-num}
\bibliography{main_pr}

\end{document}

%% file: src/1_abstract.tex
\begin{abstract}
Neural Stochastic Differential Equations (Neural SDEs) provide a principled framework for modeling continuous-time stochastic processes and have been widely adopted in fields ranging from physics to finance. Recent advances suggest that Generative Adversarial Networks (GANs) offer a promising solution to learning the complex path distributions induced by SDEs. However, a critical bottleneck lies in designing a discriminator that faithfully captures temporal dependencies while remaining computationally efficient. Prior works have explored Neural Controlled Differential Equations (CDEs) as discriminators due to their ability to model continuous-time dynamics, but such architectures suffer from high computational costs and exacerbate the instability of adversarial training. To address these limitations, we introduce \textbf{HGAN-SDEs}, a novel GAN-based framework that leverages Neural Hermite functions to construct a structured and efficient discriminator. Hermite functions provide an expressive yet lightweight basis for approximating path-level dynamics, enabling both reduced runtime complexity and improved training stability. We establish the universal approximation property of our framework for a broad class of SDE-driven distributions and theoretically characterize its convergence behavior. Extensive empirical evaluations on synthetic and real-world systems demonstrate that HGAN-SDEs achieve superior sample quality and learning efficiency compared to existing generative models for SDEs.

\end{abstract}

%% file: src/2_introduction.tex
% Content for 2_introduction_pr.tex
\vspace{-1mm}
\section{Introduction}
\label{sec:introduction} % Added a label for consistency if needed

Random fluctuations are ubiquitous in complex dynamical systems across fields such as finance~\cite{xu2024plutus,xu2025finripple}, physics~\cite{ph1}, and biology~\cite{bio1}. To capture these dynamics, Stochastic Differential Equations (SDEs)~\cite{ito1944, stratonovich1964new} serve as a powerful tool for rigorous modeling and analysis. However, they rely on predefined drift and diffusion functions, which are rarely observable and cannot be specified in advance. This mismatch limits their applicability in real-world modeling. Neural stochastic differential equations (Neural SDEs)~\cite{tzen2019neural} address this issue by integrating neural networks with SDEs, enabling drift and diffusion dynamics to be learned from data.

Training Neural SDEs is notoriously difficult because exact likelihoods are intractable. 
GANs~\cite{kidger2021neural} provide a likelihood-free alternative by matching trajectory distributions through a parametric discriminator. 
Prior work has employed CDE-based discriminators, but such approaches are often unstable and computationally expensive, limiting scalability. 
We therefore seek a discriminator design that remains likelihood-free while incorporating tractable approximations of SDE dynamics. 
To this end, we propose a novel GAN framework based on neural Hermite functions. 
Hermite functions approximate SDE transition probabilities, enabling the discriminator to effectively assess whether a generated path aligns with the underlying DGP. In addition, their statistical properties enhance optimization stability and computational efficiency.

Our contributions are threefold:
(i) we introduce HGAN-SDEs, which leverage Hermite-based approximations for reliable path discrimination;
(ii) we provide theoretical analysis showing that our framework is mathematically rigorous and can approximate arbitrary SDE solutions under suitable assumptions; and
(iii) we conduct comprehensive experiments demonstrating that our framework achieves high accuracy in fitting complex dynamic systems while maintaining reasonable computational cost.

%% file: src/3_relatedwork.tex
\vspace{-1mm}
\section{Related Work}
\label{sec:related_work}

Neural stochastic differential equations (Neural SDEs) extend neural ordinary differential equations (Neural ODEs)~\cite{chen2018neural} by adding stochastic diffusion terms, which improve long-horizon stability and enable modeling of random perturbations~\cite{tzen2019neural,Liu2019neural}. They have since been applied in domains such as finance~\cite{cuchiero2020generative}, scientific machine learning~\cite{rackauckas2020universal}, manifold learning~\cite{zeng2024latent}, and generative modeling~\cite{song2020score}. Training, however, remains challenging: stochasticity in both the dynamics and gradient estimates hinders convergence and increases computational cost~\cite{li2020scalable}. Existing approaches address these challenges along two main directions: \emph{non-adversarial} (variational or kernel-based) and \emph{adversarial} (GAN-based). Non-adversarial methods target either endpoint distributions or entire trajectories, employing reparameterized variational objectives~\cite{li2020scalable}, signature kernels~\cite{issa2024non}, or low-variance gradient estimators~\cite{course2024amortized,xu2022infinitely}. Adversarial methods instead match real and generated paths via discriminators, and have shown promise in capturing complex or heavy-tailed dynamics, particularly in financial applications~\cite{cuchiero2020generative,Gierjatowicz2020robust}. Nevertheless, non-adversarial methods often rely on restrictive approximations, while adversarial ones suffer from instability and high computational cost—highlighting the need for more efficient training frameworks.

%% file: src/4_method.tex
% Content for 4_method_pr.tex
\vspace{-1mm}
\section{Methodology}
This section outlines the key components of HGAN-SDEs. Section~\ref{sec:gener} introduces the generator based on Neural SDEs. Section~\ref{sec:hermite} presents Hermite functions as approximate estimates for SDEs, discusses their properties, and then extends them to neural Hermite functions serving as the GAN discriminator.

\subsection{Neural SDEs as the GAN Generators}
\label{sec:gener}
In our framework, the GAN generator is a Neural SDE that maps random noise 
\(v \sim \mathcal{N}(0, I_v)\) to a trajectory \(\{x(t): t \in [0,T]\}\). Formally,
\(
G_\theta(v): \quad x(0)=h_\theta(v), \quad 
dx(t)=f_\theta(t,x(t))\,dt+g_\theta(t,x(t))\circ dw(t),
\)
where \(h_\theta\) maps noise to the initial state, and 
\(f_\theta,g_\theta\) (Lipschitz neural networks) parameterize drift and 
diffusion. The trajectory is obtained by numerically solving the SDE.  

The discriminator \(D_\phi\) operates on entire trajectories,
\(
D_\phi(x) = d_\phi(\{x(t)\}) \in \mathbb{R},
\)
where \(d_\phi\) is a parameterized function, later instantiated by Neural Hermite functions.  Training follows the WGAN objective
\(
\min_\theta \max_\phi \, \mathbb{E}_{v}\big[D_\phi(G_\theta(v))\big] 
- \mathbb{E}_{\mu}\big[D_\phi(\mu)\big],
\)
with \(\mu\) denoting a real trajectory.

\subsection{Neural Hermite Functions as GAN Discriminators}
\label{sec:hermite}
We design GAN discriminators based on Neural Hermite functions, an orthogonal polynomial basis well-suited for approximating complex path distributions in SDEs.
We first establish the rationale by introducing Hermite expansions for SDE transition densities, and then analyze the expressivity of Hermite-based discriminators.

\noindent\textbf{Hermite Functions for SDE Transition Density Estimation.} Beyond serving as expressive bases for discriminators, Hermite functions offer analytical advantages for characterizing the transition densities of SDEs. Consider a general SDE of the form:
\(
dx(t) = f_\theta(t, x(t))\, dt + g_\theta(t, x(t)) \circ dw(t),
\)
where the drift $f_\theta$ and diffusion $g_\theta$ are parameterized by $\theta$, and $\circ$ denotes the Stratonovich integral. The evolution of the transition density $p(x(t)\mid x(0)=x_0, t)$ is governed by the Fokker-Planck equation:
\[
\begin{aligned}
    \frac{\partial p(x, t)}{\partial t} 
    &= \frac{1}{2} \frac{\partial^2}{\partial x^2} \left[g_\theta^2(t, x) \, p(x, t)\right]
    - \frac{\partial}{\partial x} \left[f_\theta(t, x) \, p(x, t)\right],
\end{aligned}
\]
where we abbreviate $p(x, t) := p(x(t)=x \mid x(0) = x_0, t)$. The solution to this partial differential equation characterizes the probability flow induced by the SDE. Now, consider the case where the transition density reaches stationarity. That is, as $t \to \infty$, the distribution $p(x, t)$ converges to a time-independent stationary density:
\(
p(x) := \lim_{t \to \infty} p(x, t),
\quad\text{with}\quad \frac{\partial p(x)}{\partial t} = 0.
\)
Under stationarity, the Fokker-Planck equation simplifies to:
\begin{equation}
\label{eq:fp}
    \frac{1}{2} \frac{d^2}{d x^2} \left[g_\theta^2(x)\, p(x)\right]
    - \frac{d}{d x} \left[f_\theta(x)\, p(x)\right] = 0,
\end{equation}
where we have omitted explicit dependence on $t$ for clarity.

To approximate this stationary density $p(x)$, we propose using Hermite functions as a basis expansion:
\(
p(x) \approx \sum_{n=0}^N c_n \psi_n(x),
\)
where the $\psi_n(x)$ are the Hermite functions defined by:
\[
\psi_n(x) = \frac{1}{\sqrt{2^n n! \sqrt{\pi}}} e^{x^2 / 2} (-1)^n \frac{d^n}{dx^n} e^{-x^2}.
\]
Hermite functions enjoy a critical property: they form a complete orthonormal basis in the Hilbert space $L^2(\mathbb{R}, dx)$. This property guarantees that any square-integrable function—including valid probability densities—can be approximated arbitrarily well by a linear combination of these functions. We formalize this in the following theorem:

\begin{tcolorbox}[colback=lightpurple, colframe=lightpurple, sharp corners=all, boxrule=0mm, boxsep=0.5mm, left=1.5mm, right=1.5mm, top=1.5mm, bottom=1.5mm]
\begin{theorem}[Orthonormality of Hermite functions~\cite{abramowitz1972handbook,eijndhoven1990new}]
\label{thm:orthonormal}
    The Hermite functions $\{\psi_n(x)\}_{n=0}^\infty$ satisfy the orthonormality condition:
    \(
        \int_{\mathbb{R}} \psi_n(x) \psi_m(x) dx = \delta_{nm},
    \)
    where $\delta_{nm}$ is the Kronecker delta.
\end{theorem}
\end{tcolorbox}
\vspace{-0.2em}

The orthonormality of Hermite functions enables stable and interpretable density expansions, making the discriminator both expressive and sensitive to subtle path differences.

In this work, we approximate the stationary transition density of an SDE as a truncated Hermite expansion of order $N$, formulated as:
\begin{equation}
\label{eq:basis}
\begin{aligned}
    p(x(t) \mid x(0) = x_0) &\approx \sum_{n=0}^N c_n(x_0, x_t, t) \, e^{-x^2 / 2} H_n(x),
\end{aligned}
\end{equation}
where \(\psi_n(x)\) denotes the \(n\)-th Hermite function, and \(H_n(x)\) is the corresponding Hermite polynomial. The time-varying coefficients \(c_n(x_0, x_t, t)\) depend on the initial state and the evolution governed by the drift and diffusion functions. These coefficients can, in principle, be derived by projecting the solution onto the Hermite basis.

As an illustrative case, consider the Ornstein–Uhlenbeck (OU) process:
\(
dx(t) = -\theta x(t)\, dt + \sigma\, dw(t),
\)
whose solution admits a Hermite expansion with time-evolving coefficients \(c_n(t)\) given by:
\(
c_n(t) = c_n(0) e^{-\theta n t},
\)
where \(c_n(0)\) is the initial projection of \(x(0)\) onto the Hermite function \(\psi_n(x)\). This result highlights a key smoothing property: higher-order components decay exponentially, leading to increasingly regular trajectories as time evolves. We formally justify the Hermite approximation through the following theorem. First, we specify the regularity conditions for \(f_\theta\) and \(g_\theta\).
\begin{assumption}
\label{assm:mom}
    The functions \(f_\theta(t, x(t))\) and \(g_\theta(t, x(t))\) are Lipschitz continuous in \(x\) and have bounded moments of all orders.
\end{assumption}
\begin{tcolorbox}[colback=lightpurple, colframe=lightpurple, sharp corners=all, boxrule=0mm, boxsep=0.5mm, left=1.5mm, right=1.5mm, top=1.5mm, bottom=1.5mm]
\begin{theorem}[Convergence of Hermite Expansion]
\label{thm:basis}
    Under Assumption~\ref{assm:mom}, the Hermite expansion in Equation~\eqref{eq:basis} converges in \(L^2(\mathbb{R})\) as \(N \to \infty\). That is, for any square-integrable stationary transition density \(p(x)\),
    \(
        \lim_{N \to \infty} \left\| p(x) - \sum_{n=0}^{N} c_n \psi_n(x) \right\|_{L^2} = 0.
    \)
\end{theorem}
\end{tcolorbox}
\vspace{-0.3em}
\begin{proof}[Proof Sketch]
Following \cite{eijndhoven1990new}, integrating both sides of (\ref{eq:fp}) leads to an identity involving a constant $C$. 
A further integration over the real line shows that the integrals remain finite, which forces $C=0$. 
Hence we obtain the reduced form
\[
\frac{1}{2}\frac{\partial}{\partial x}\!\left[g_\theta^2(t,x)\,p(x(t)\mid x_0)\right]
- f_\theta(t,x)\,p(x(t)\mid x_0) = 0.
\]

This equation yields a first-order relation for $\partial_x p(x(t)\mid x_0)$. 
Using the separation ansatz $p(x(t)\mid x_0)=\alpha(t)p(x)\gamma(x)$, the Fokker–Planck equation decouples into a temporal part
\(
\frac{\partial}{\partial t}\alpha(t) = -\lambda \alpha(t),
\)
and a spatial eigenvalue problem of Sturm–Liouville type:
\[
\frac{1}{2}\frac{\partial}{\partial x}\!\left[g_\theta^2(t,x)\,p(x)\,\frac{\partial}{\partial x}\gamma(x)\right] 
= -\lambda p(x)\gamma(x).
\]

By Sturm–Liouville theory, the eigenfunctions $\gamma_n$ form a $p$-weighted orthonormal basis of $L^2(\mathbb{R},p(x)dx)$, and the solution admits 
\(
p(x(t)=x \mid x_0) 
= \sum_{n=0}^\infty e^{-\lambda_n t}\, \gamma_n(x_0)\, p(x)\,\gamma_n(x).
\)
This matches the structure of (\ref{eq:basis}).
\end{proof}

\begin{table*}[!t]
\centering
\renewcommand{\arraystretch}{1.1}
\resizebox{\textwidth}{!}{
  \begin{tabular}{ll cccc cccc cccc cccc}
  \toprule
  \multirow{2}{*}{\textbf{Category}} & \multirow{2}{*}{\textbf{Model}}
  & \multicolumn{4}{c}{\textbf{GBM}}
  & \multicolumn{4}{c}{\textbf{OU}}
  & \multicolumn{4}{c}{\textbf{CIR}}
  & \multicolumn{4}{c}{\textbf{Polynomial Drift}} \\
  \cmidrule(lr){3-6} \cmidrule(lr){7-10} \cmidrule(lr){11-14} \cmidrule(lr){15-18}
  & & \textbf{MISE} & \textbf{TD} & \textbf{MSE} & \textbf{MMD}
    & \textbf{MISE} & \textbf{TD} & \textbf{MSE} & \textbf{MMD}
    & \textbf{MISE} & \textbf{TD} & \textbf{MSE} & \textbf{MMD}
    & \textbf{MISE} & \textbf{TD} & \textbf{MSE} & \textbf{MMD} \\
  \midrule
\textbf{Linear Model}
& DLinear~\cite{zeng2022dlinear}     
& 6.00 & 3.20 & 0.09 & 0.11 
& 61.61 & 9.09 & 0.52 & 0.31 
& 196.79 & 7.01 & 1.34 & 0.44 
& 5.18 & 3.86 & 0.08 & 0.13 \\
\midrule
\multirow{2}{*}{\textbf{Convolutional Models}}
& MICN~\cite{wang2023micn}        
& 7.05 & 4.74 & 0.12 & 0.16 
& 159.01 & 9.55 & 1.97 & 0.41 
& 33.74 & 6.87 & 0.59 & 0.23 
& 7.65 & 4.48 & 0.11 & 0.12 \\
& TimesNet~\cite{wu2023timesnet}    
& 1.50 & 0.21 & 0.03 & \cellcolor{orange!30}0.02 
& 13.56 & 8.10 & 0.14 & 0.15 
& 2.42 & 2.70 & 0.05 & 0.07 
& 1.86 & 1.74 & 0.04 & 0.04 \\
\midrule
\multirow{3}{*}{\textbf{Attention-based Models}}
& Transformer~\cite{vaswani2017attention}   
& 0.74 & 0.55 & \cellcolor{orange!30}0.02 & 0.03 
& \cellcolor{orange!30}1.22 & \cellcolor{orange!30}0.70 & \cellcolor{orange!30}0.03 & \cellcolor{orange!70}\textbf{0.02} 
& \cellcolor{orange!30}1.44 & \cellcolor{orange!30}0.93 & \cellcolor{orange!30}0.03 & \cellcolor{orange!30}0.04 
& 9.47 & 10.00 &0.23 & 0.40 \\
& Informer~\cite{zhou2021informer}      
& 3.53 & 3.49 & 0.05 & 0.10 
& 8.58 & 3.20 & 0.10 & 0.09 
& 1.57 & 2.46 & 0.04 & 0.07 
& 67.67 & 7.13 & 0.99 & 0.41 \\
& Autoformer~\cite{wu2021autoformer}    
& 1.71 & 1.53 & 0.03 & 0.06 
& 10.87 & 9.22 & 0.13 & 0.15 
& 2.39 & 4.41 & 0.05 & 0.08 
& 145.31 & 6.33 & 1.05 & 0.33 \\
\midrule
\textbf{RNN Model}
& SegRNN~\cite{wang2023segrnn}        
& 1.09 & 0.31 & 0.02 & 0.02 
& 13.68 & 7.24 & 0.16 & \cellcolor{orange!30}0.14 
& 2.12 & 2.11 & 0.04 & 0.06 
& 4.98 & 7.55 & 0.09 & 0.28 \\
\midrule
\textbf{State Space Model}
& Mamba~\cite{gu2023mamba}         
& 1.04 & \cellcolor{orange!70}\textbf{0.10} & \cellcolor{orange!30}0.02 & \cellcolor{orange!70}\textbf{0.01} 
& 13.72 & 7.58 & 0.17 & \cellcolor{orange!30}0.14 
& 2.39 & 2.55 & 0.05 & 0.07 
& 1.35 & 2.21 & 0.03 & 0.06 \\
\midrule
\multirow{3}{*}{\textbf{SDE-based Models}}
& Latent-SDEs~\cite{rubanova2019latent}    
& 0.44 & 10.00 & \cellcolor{orange!70}\textbf{0.01} & 0.42 
& 101.40 & 10.00 & 2.51 & 0.93 
& 39.78 & 10.00 & 0.98 & 0.86 
& 3.73 & 6.81 & 0.06 & 0.20 \\
& GAN-SDEs~\cite{kidger2021neural}       
& \cellcolor{orange!30}0.38 & 9.27 & \cellcolor{orange!70}\textbf{0.01} & 0.32 
& 214.54 & 10.00 & 4.11 & 0.88 
& 209.23 & 10.00 & 3.91 & 0.87 
& \cellcolor{orange!30}0.11 & \cellcolor{orange!30}0.27 & \cellcolor{orange!30}0.02 & \cellcolor{orange!30}0.02 \\
& HGAN-SDEs 
& \cellcolor{orange!70}\textbf{0.17} & \cellcolor{orange!30}0.12 & \cellcolor{orange!70}\textbf{0.01} & \cellcolor{orange!70}\textbf{0.01} 
& \cellcolor{orange!70}\textbf{0.89} & \cellcolor{orange!70}\textbf{0.33} & \cellcolor{orange!70}\textbf{0.01} & \cellcolor{orange!70}\textbf{0.02} 
& \cellcolor{orange!70}\textbf{0.76} & \cellcolor{orange!70}\textbf{0.18} & \cellcolor{orange!70}\textbf{0.02} & \cellcolor{orange!70}\textbf{0.01} 
& \cellcolor{orange!70}\textbf{0.04} & \cellcolor{orange!70}\textbf{0.11} & \cellcolor{orange!70}\textbf{0.01} & \cellcolor{orange!70}\textbf{0.01} \\
  \bottomrule
  \end{tabular}
}
\vspace{-0.6em}
\caption{Comparison of MISE, TD, MSE, and MMD on GBM, OU, CIR, and Polynomial Drift datasets. 
Lower values indicate better performance. 
\colorbox{orange!70}{Darker orange} highlights the best results, 
\colorbox{orange!30}{lighter orange} marks the second-best. 
For clarity, MSE values below $0.01$ are shown in scientific notation where appropriate. 
The HGAN-SDEs use a 3-layer MLP discriminator with hidden size 128. 
All hyperparameters have been thoroughly tuned via grid search to ensure fairness.}
\label{tab:simulation_combined}
\vspace{-0.5em}
\end{table*}

\begin{table*}[h]
\centering
\resizebox{0.93\linewidth}{!}{%
\begin{tabular}{ll}
\toprule
\textbf{Data Generation Process} & \textbf{Parameters } \\
\midrule
$dX_t = \mu dt + \sigma dW_t$ 
& $\mu = 0.05$, $\sigma = 0.02$, $x_0 = 20.0 \pm 0.1$ \\
\midrule
$dX_t = \kappa(\alpha - X_t)dt + \sigma dW_t$ 
& $\kappa = 0.0658$, $\alpha = 23.0$, $\sigma = 0.2213$, $x_0 = 20.0 \pm 0.1$ \\
\midrule
$dX_t = \kappa(\alpha - X_t)dt + \sigma \sqrt{X_t}\, dW_t$ 
& $\kappa = 0.0145$, $\alpha = 23.0$, $\sigma = 0.06521$, $x_0 = 20.0 \pm 0.1$ \\
\midrule
$dX_t = (\alpha_{-1}X_t^{-1} + \alpha_0 + \alpha_1 X_t + \alpha_2 X_t^2)dt + \sigma X_t^{3/2} dW_t$ 
& $\alpha_{-1} = 0.01$, $\alpha_0 = 0.01$, $\alpha_1 = 0.001$, $\alpha_2 = -4.604$, $\sigma = 0.1$, $x_0 = 20.0 \pm 0.1$ \\
\bottomrule
\end{tabular}%
}
\vspace{-0.5em}
\caption{Parameter Settings of Simulation Data: GBM, OU, CIR, and Polynomial Drift.}
\label{tab:parameters}
\end{table*}
\vspace{-0.5em}

\subsection{Expressivity of Hermite-Based Discriminator}

To ensure our Hermite-based discriminator can distinguish \(\mu\) from \(\nu\), we analyze the expressivity of its function class. Each discriminator is modeled as a linear combination of Hermite basis functions:
\(
\mathcal{P} := \text{span} \{\psi_n(x)\}_{n=0}^N.
\)
We now consider the integral probability metric induced by \(\mathcal{P}\):
\(
d_{\mathcal{P}}(\mu, \nu) := \sup_{f \in \mathcal{P}} \left| \mathbb{E}_{x \sim \mu}[f(x)] - \mathbb{E}_{x \sim \nu}[f(x)] \right|.
\)
\begin{tcolorbox}[colback=lightpurple, colframe=lightpurple, sharp corners=all, boxrule=0mm, boxsep=0.5mm, left=1.5mm, right=1.5mm, top=1.5mm, bottom=1.5mm]
\begin{theorem}[Discriminative Power of Hermite Span]
\label{the:dis_app}
If Assumption~\ref{assm:mom} holds and the true transition density \(p(x(t) \mid x(0) = x_0, t)\) is square-integrable, then
\[
d_{\mathcal{P}}(\mu, \nu) = 0 \Longleftrightarrow \mu = \nu.
\]
\end{theorem}
\end{tcolorbox}
\vspace{-0.2em}
\begin{proof}
According to~\cite{stekloff1916théorème}, Hermite polynomials \(\{H_n(x)\}\) are dense in \(L^2(\mathbb{R}, e^{-x^2} dx)\), implying that \(\{\psi_n(x)\}\) are dense in \(L^2(\mathbb{R}, dx)\). By Theorem 2.2 of~\cite{DBLP:conf/iclr/Zhang0ZX018}, the induced metric \(d_{\mathcal{P}}(\mu, \nu)\) is therefore discriminative.
\end{proof}

%% file: src/5_results.tex
% Content for 5_results_pr.tex
\vspace{-1mm}
\section{Experiments}
This section examines the capabilities of HGAN-SDEs. Section~\ref{sec:experment_setting} outlines benchmark datasets (simulation and real-world), baselines, and evaluation metrics. Section~\ref{sec:analysis} presents performance comparisons with competing methods, and Section~\ref{sec:exploration} further investigates Hermite truncation effects and scale–performance trade-offs.

\subsection{Experiment Settings}
\label{sec:experment_setting}
\noindent\textbf{Benchmarks.} Following Yacine's work~\cite{ait1999transition}, we construct four simulation datasets to evaluate our framework: geometric Brownian motion (GBM), Ornstein-Uhlenbeck (OU), Cox-Ingersoll-Ross (CIR), and Polynomial Drift. Their parameter configurations are summarized in Table~\ref{tab:parameters}.  Each dataset consists of 20,000 training and 6,000 test samples, formatted for sequence prediction: the first 100 points are given as inputs and the subsequent 50 as prediction targets. For real-world validation, we use three datasets. Stock-AAL and Stock-ADBE each contain 43,146 five-minute stock prices (2017–2018) with several hundred missing values, reflecting the stochastic volatility of financial markets. The Traffic dataset comprises 21,600 five-minute records of Beijing traffic flows in 2022, with no missing values and strong periodicity.

\noindent\textbf{Evaluation Metrics.} We evaluate the proposed method using four complementary metrics: Mean Integrated Squared Error (MISE) for global accuracy, Tail Difference (TD) for tail fidelity beyond the 5\% quantile, Mean Squared Error (MSE) for pointwise error, and Maximum Mean Discrepancy (MMD) for distributional alignment.

\subsection{Experimental Results Analysis}

\label{sec:analysis} % Moved label here from below the figure

\noindent\textbf{Simulation Experiment Analysis.} As shown in Table~\ref{tab:simulation_combined}, single linear or convolutional architectures are inadequate for modeling stochastic systems, as they fail to capture nonlinear diffusion dynamics. As shown in Figure~\ref{fig:changeoveertime}, attention-based models achieve strong performance in some regimes but lack the inductive bias for stochastic diffusion, leading to instability in more complex systems. GAN-SDEs with CDE-based discriminators further suffer from high gradient variance, causing unstable training and performance collapse. In contrast, HGAN-SDEs deliver both stability and generalization, consistently achieving state-of-the-art results across all synthetic benchmarks.

\begin{figure}[!htbp]
    \centering
    \resizebox{1\columnwidth}{!}{
        \includegraphics{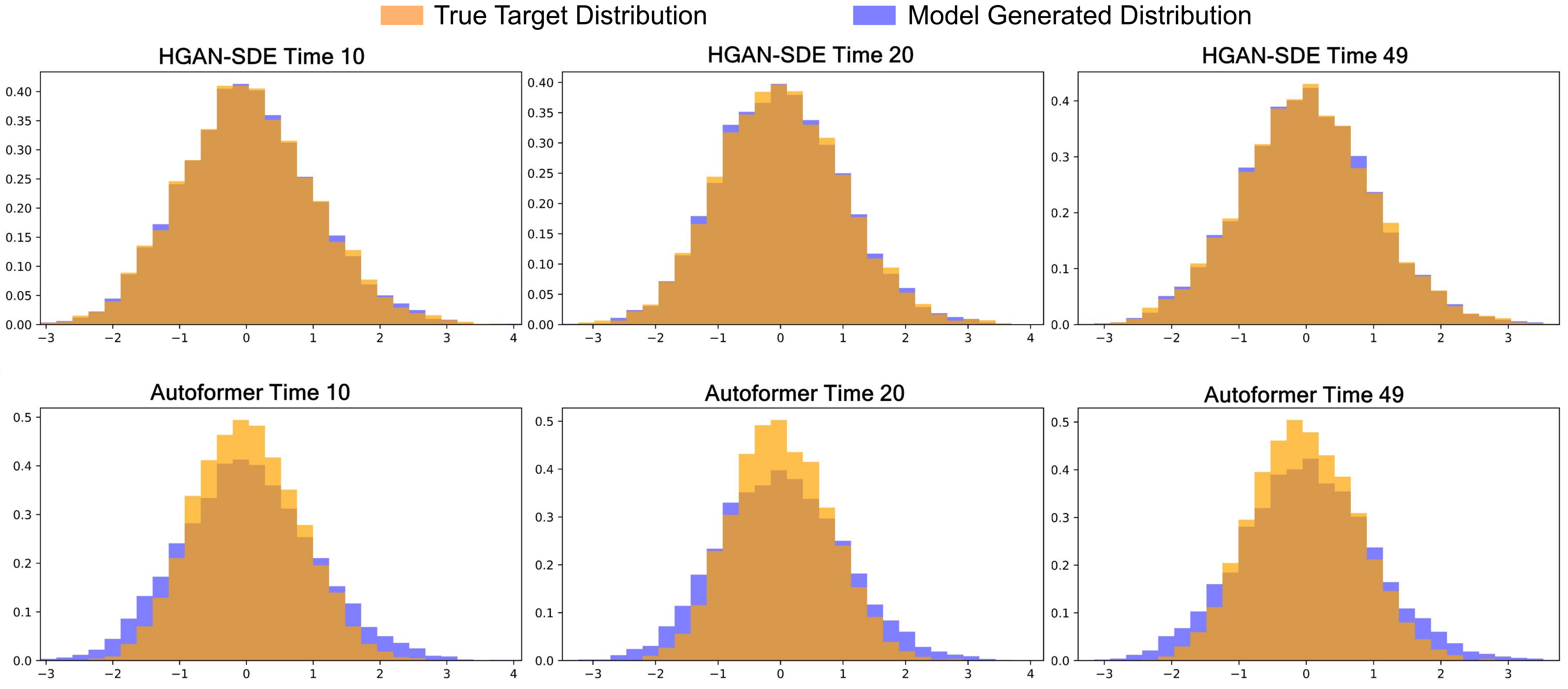}
    }
    \vspace{-0.6em}
    \caption{Comparison between HGAN-SDEs and Autoformer on the CIR dataset. 
The figure shows how predicted sequence distributions evolve across time points, 
revealing that HGAN-SDEs maintain smoother and more consistent dynamics, while Autoformer tends to generate more dispersed outputs.}
    \label{fig:changeoveertime} % This label was used twice, assuming this is for CIR comparison. The other figure with same label is for MISE vs Hermite.
    \vspace{-0.5em}
\end{figure}

\begin{table*}[htbp]
\centering
\renewcommand{\arraystretch}{1.1}
\resizebox{2\columnwidth}{!}{
\begin{tabular}{l l cccc cccc cccc}
\toprule
\multirow{2}{*}{\textbf{Category}} & \multirow{2}{*}{\textbf{Model}} & 
\multicolumn{4}{c}{\textbf{Stock-AAL}} & 
\multicolumn{4}{c}{\textbf{Stock-ADBE}} & 
\multicolumn{4}{c}{\textbf{Traffic}} \\
\cmidrule(lr){3-6} \cmidrule(lr){7-10} \cmidrule(lr){11-14}
 & & \textbf{MISE} & \textbf{TD} & \textbf{MSE} & \textbf{MMD} & 
 \textbf{MISE} & \textbf{TD} & \textbf{MSE} & \textbf{MMD} & 
 \textbf{MISE} & \textbf{TD} & \textbf{MSE} & \textbf{MMD} \\
\midrule
\textbf{Linear Model} & DLinear & 158.75 & 10.00 & 1.59 & 1.00 & 158.61 & 8.80 & 1.59 & 0.88 & 8.99 & 3.60 & \cellcolor{orange!30}0.09 & 0.36 \\
\midrule
\multirow{2}{*}{\textbf{Convolution-based Models}} & MICN & 84.81 & 10.00 & 0.85 & 1.00 & 193.65 & 10.00 & 1.94 & 1.00 & 9.66 & 3.73 & 0.10 & 0.37 \\
& TimesNet & 93.00 & 10.00 & 0.93 & 1.00 & 89.89 & 3.99 & 0.90 & 0.40 & 15.52 & 4.26 & 0.16 & 0.43 \\
\midrule
\multirow{2}{*}{\textbf{Attention-based Models}} & Transformer & 59.80 & \cellcolor{orange!70}\textbf{0.16} & 0.60 & \cellcolor{orange!70}\textbf{0.02} & 485.93 & \cellcolor{orange!70}0.37 & 4.86 & \cellcolor{orange!70}\textbf{0.04} & 1885.06 & 10.01 & 18.85 & 1.00 \\
& Autoformer & 112.45 & 9.98 & 1.12 & 1.00 & \cellcolor{orange!30}8.40 & 5.21 & \cellcolor{orange!30}0.08 & 0.52 & 88.31 & 3.49 & 0.88 & 0.35 \\
\midrule
\textbf{RNN Model} & SegRNN & 156.97 & 10.00 & 1.57 & 1.00 & 53.65 & \cellcolor{orange!30}1.81 & 0.54 & 0.18 & 28.23 & \cellcolor{orange!30}2.58 & 0.28 & \cellcolor{orange!30}0.26 \\
\midrule
\textbf{State-space Model} & Mamba & 73.31 & 7.82 & 0.73 & 0.78 & 45.35 & 6.26 & 0.45 & 0.63 & \cellcolor{orange!70}\textbf{0.00} & 3.59 & \cellcolor{orange!70}\textbf{0.00} & 0.36 \\
\midrule
\multirow{3}{*}{\textbf{Differential Equation-based Models}} & Latent-SDEs & \cellcolor{orange!30}35.52 & 10.00 & \cellcolor{orange!30}0.36 & 1.00 & 34.67 & 10.00 & 0.35 & 1.00 & 32.17 & 9.50 & 0.32 & 0.95 \\
& GAN-SDEs & 38.00 & 10.01 & 0.38 & 1.00 & 37.89 & 10.01 & 0.38 & 1.00 & 35.79 & 9.74 & 0.36 & 0.97 \\
& \textbf{HGAN-SDEs} & \cellcolor{orange!70}\textbf{2.89} & \cellcolor{orange!30}0.36 & \cellcolor{orange!70}\textbf{0.03} & \cellcolor{orange!30}0.04 & \cellcolor{orange!70}\textbf{8.05} & 3.59 & \cellcolor{orange!70}\textbf{0.08} & \cellcolor{orange!30}\textbf{0.36} & \cellcolor{orange!30}0.22 & \cellcolor{orange!70}\textbf{0.17} & \cellcolor{orange!70}\textbf{0.00} & \cellcolor{orange!70}\textbf{0.02} \\
\bottomrule
\end{tabular}
}
\vspace{-0.6em}
\caption{Performance comparison of models on Stock-AAL, Stock-ADBE, and Traffic datasets. Lower values are better for all metrics (MISE, TD, MSE, MMD). 
\colorbox{orange!70}{Darker orange} marks the best result, 
\colorbox{orange!30}{lighter orange} marks the second-best. 
All MSE values are scaled by $10^{-3}$.}
\label{tab:final_table_stock_traffic}
\end{table*}

\noindent\textbf{Real-World Experiment Analysis.} Our results (Table~\ref{tab:final_table_stock_traffic}) show that model performance depends on data characteristics. Linear and convolutional models handle regular patterns such as traffic flows but fail in non-stationary systems like finance, where volatility and high-frequency noise dominate. In contrast, HGAN-SDEs remain stable and achieve the best results across all benchmarks by combining a Neural SDE generator for deterministic and stochastic dynamics with a Hermite-based discriminator for trajectory discrimination. This design enables robust adaptation to both regular and highly stochastic systems.

\subsection{More Explorations}
\label{sec:exploration}

\noindent\textbf{The Number of Neural Hermite Function Terms.}
Increasing the number of neural Hermite terms improves transition density approximation, aiding accurate modeling of stochastic processes. However, each additional term introduces extra networks and parameters, complicating training and risking longer runtimes, convergence issues, and instability. As shown in Figure~\ref{fig:mise_vs_hermite_terms}, three to four terms suffice, highlighting the need to balance approximation accuracy with architectural simplicity and training stability.

\begin{figure}[!htbp]
    \centering
    \resizebox{0.8\columnwidth}{!}{
        \includegraphics{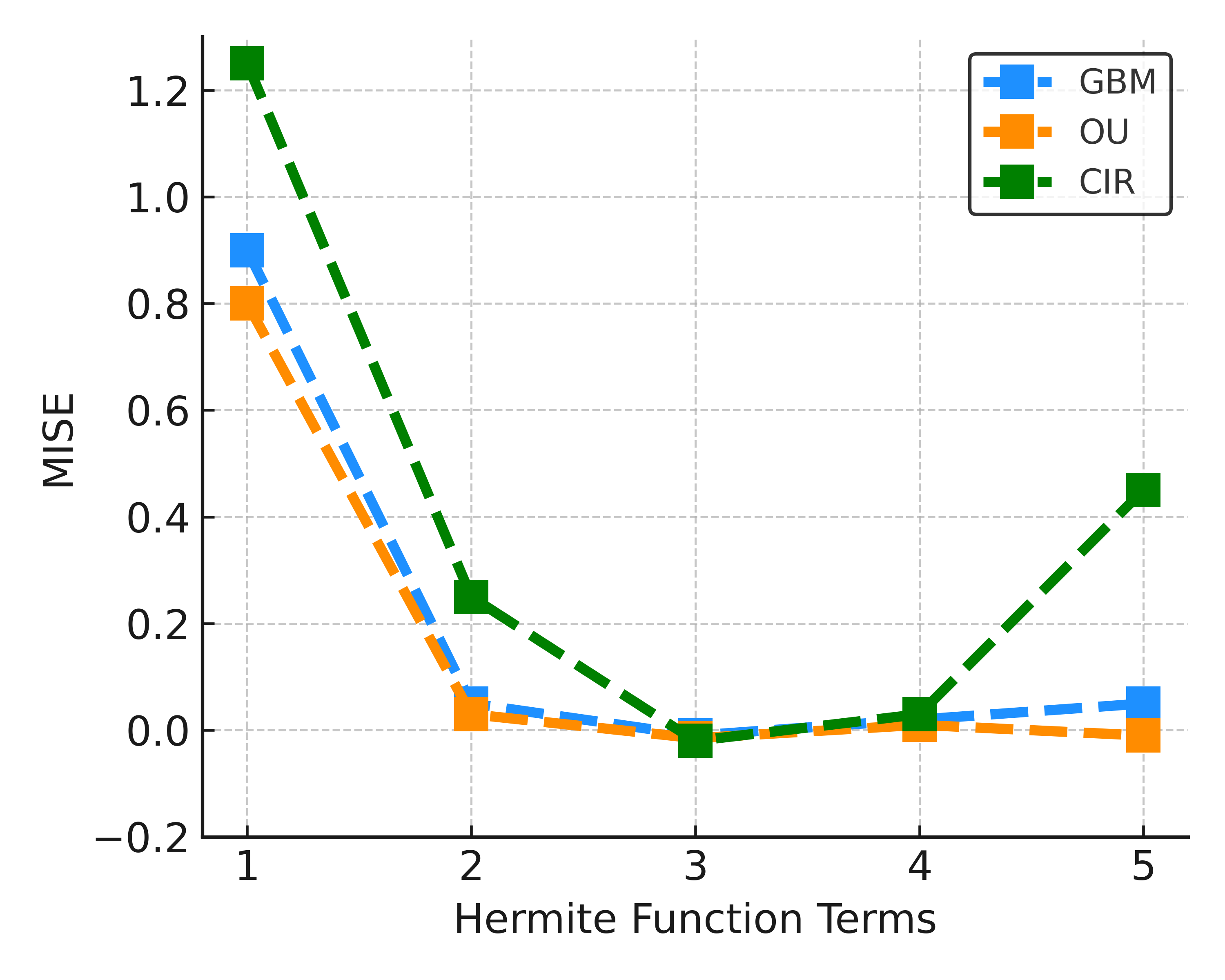}
    }
    \vspace{-0.5em}
    \caption{The figure shows the changes in MISE with the increase of neural Hermite terms in HGAN-SDEs. We present the patterns on three datasets, with the parameter configurations consistent with the previous experiments.}
    \label{fig:mise_vs_hermite_terms} % Changed label to be unique, original was fig:changeoveertime
\end{figure}

\noindent\textbf{Computational Efficiency.} As shown in Table~\ref{tab:gan_sde_comparison}, HGAN-SDEs achieves the fastest convergence with the lowest resource cost, requiring only 1.2 GPU-hours on average compared to 10.0 for CDE and 6.5 for LSTM. This highlights its practicality in resource-constrained or large-scale settings. To ensure fair comparison, we kept parameter counts as consistent as possible across models. All generators share the same Neural SDE architecture, with drift and diffusion functions parameterized by 2–3 layer MLPs. For HGAN-SDEs, we use a 3-layer MLP with hidden size 128 to parameterize the Hermite basis. Other discriminators (LSTM, MLP, CDE) are implemented with comparable parameter scales.
\begin{table}[!htbp]
\centering
\label{tab:gan_sde_comparison}
\resizebox{0.95\columnwidth}{!}{
\begin{tabular}{lccccc}
\toprule
\multirow{2}{*}{Model} & \multicolumn{4}{c}{Polynomial Drift} & \multirow{2}{*}{GPU-h} \\
\cmidrule(lr){2-5}
 & MISE & TD & MSE & MMD & \\
\midrule
LSTM Discriminator & 1.12 & 0.43 & 0.03 & 0.04 & 6.5 \\
MLP Discriminator  & 0.95 & 0.37 & 0.03 & 0.04 & 4.8 \\
GAN-SDEs (CDE)      & 0.11 & 0.27 & 0.02 & 0.02 & 10.0 \\
HGAN-SDEs (Hermite)& 0.04 & 0.11 & 0.01 & 0.01 & 1.2\\
\bottomrule
\end{tabular}
}
\vspace{-0.5em}
\caption{Comparison of performance and training cost on the Polynomial Drift benchmark. All models share the same Neural SDE generator.}
\end{table}

\vspace{-1mm}
\section{Conclusion}
We propose HGAN-SDEs, a GAN framework combining Neural SDE generators with Hermite-function discriminators. Leveraging orthogonal Hermite bases reduces complexity, stabilizes training, and preserves discriminative power. Across synthetic and real-world experiments, HGAN-SDEs surpass linear, convolutional, transformer, and SDE-based baselines in accuracy and efficiency, converging faster and more robustly with only a few Hermite terms.

%% file: main_pr.bib
@String(ICLR= {Proc. ICLR})

@String(ACL = {Proc. ACL})

@String(NIPS  = {Proc. NeurIPS})

@String(ICML  = {Proc. ICML})

@inproceedings{chen2018neural,
  author = {Chen, Ricky T. Q. and Rubanova, Yulia and Bettencourt, Jesse and Duvenaud, David K},
  booktitle = NIPS,
  pages = {},
 title = {Neural Ordinary Differential Equations},
 volume = {31},
 year = {2018}
}

@inproceedings{li2020scalable,
  title={Scalable gradients and variational inference for stochastic differential equations},
  author={Li, Xuechen and Wong, Ting-Kam Leonard and Chen, Ricky TQ and Duvenaud, David K},
  booktitle = {Symp. Adv. Approx. Bayesian Inf.},
  pages={1--28},
  year={2020}
}

@article{rackauckas2020universal,
  title={Universal differential equations for scientific machine learning},
  author={Rackauckas, Christopher and Ma, Yingbo and Martensen, Julius and Warner, Collin and Zubov, Kirill and Supekar, Rohit and Skinner, Dominic and Ramadhan, Ali and Edelman, Alan},
  journal={arXiv preprint arXiv:2001.04385},
  year={2020}
}

@article{liu2019neural,
  title={Neural sde: Stabilizing neural ode networks with stochastic noise},
  author={Liu, Xuanqing and Xiao, Tesi and Si, Si and Cao, Qin and Kumar, Sanjiv and Hsieh, Cho-Jui},
  journal={arXiv preprint arXiv:1906.02355},
  year={2019}
}

@article{cuchiero2020generative,
  title={A generative adversarial network approach to calibration of local stochastic volatility models},
  author={Cuchiero, Christa and Khosrawi, Wahid and Teichmann, Josef},
  journal={Risks},
  volume={8},
  number={4},
  pages={101},
  year={2020}
}

@article{song2020score,
  title={Score-based generative modeling through stochastic differential equations},
  author={Song, Yang and Sohl-Dickstein, Jascha and Kingma, Diederik P and Kumar, Abhishek and Ermon, Stefano and Poole, Ben},
  journal={arXiv preprint arXiv:2011.13456},
  year={2020}
}

@article{gierjatowicz2020robust,
  title={Robust pricing and hedging via neural sdes},
  author={Gierjatowicz, Patryk and Sabate-Vidales, Marc and {\v{S}}i{\v{s}}ka, David and Szpruch, Lukasz and {\v{Z}}uri{\v{c}}, {\v{Z}}an},
  journal={arXiv preprint arXiv:2007.04154},
  year={2020}
}

@inproceedings{issa2024non,
  title={Non-adversarial training of Neural SDEs with signature kernel scores},
  author={Issa, Zacharia and Horvath, Blanka and Lemercier, Maud and Salvi, Cristopher},
  booktitle=NIPS,
  volume={36},
  year={2024}
}

@inproceedings{xu2022infinitely,
  title={Infinitely deep bayesian neural networks with stochastic differential equations},
  author={Xu, Winnie and Chen, Ricky TQ and Li, Xuechen and Duvenaud, David},
  booktitle={International Conference on Artificial Intelligence and Statistics},
  pages={721--738},
  year={2022},
}

@inproceedings{zeng2024latent,
  title={Latent sdes on homogeneous spaces},
  author={Zeng, Sebastian and Graf, Florian and Kwitt, Roland},
  booktitle=NIPS,
  volume={36},
  year={2024}
}

@inproceedings{course2024amortized,
  title={Amortized reparametrization: efficient and scalable variational inference for latent SDEs},
  author={Course, Kevin and Nair, Prasanth},
  booktitle=NIPS,
  volume={36},
  year={2024}
}

@inproceedings{kidger2021neural,
  title={Neural sdes as infinite-dimensional gans},
  author={Kidger, Patrick and Foster, James and Li, Xuechen and Lyons, Terry J},
  booktitle=ICML,
  pages={5453--5463},
  year={2021}
}

@article{ait1999transition,
  title={Transition densities for interest rate and other nonlinear diffusions},
  author={A{\"\i}t-Sahalia, Yacine},
  journal={The journal of finance},
  volume={54},
  number={4},
  pages={1361--1395},
  year={1999}
}

@article{ito1944,
  title={Stochastic Integral},
  author={It{\^o}, Kiyosi},
  journal={Proc. Imp. Acad.},
  volume={20},
  number={8},
  pages={519--524},
  year={1944}
}

@article{stratonovich1964new,
  title={A new form of representing stochastic integrals and equations},
  author={Stratonovich, RL},
  journal={Vestnik Moskov. Univ. Ser. I Mat. Meh},
  volume={1},
  pages={3--12},
  year={1964}
}

@article{tzen2019neural,
  title={Neural stochastic differential equations: Deep latent gaussian models in the diffusion limit},
  author={Tzen, Belinda and Raginsky, Maxim},
  journal={arXiv preprint arXiv:1905.09883},
  year={2019}
}

@inproceedings{DBLP:conf/iclr/Zhang0ZX018,
  author       = {Pengchuan Zhang and
                  Qiang Liu and
                  Dengyong Zhou and
                  Tao Xu and
                  Xiaodong He},
  title        = {On the Discrimination-Generalization Tradeoff in GANs},
  booktitle    = ICLR,
  year         = {2018}
}

@article{stekloff1916théorème,
  title={Th{\'e}or{\`e}me de fermeture pour les polynomes de Tch{\'e}bychef--Laguerre},
  author={Stekloff, Wladimir},
  journal   = {Bulletin de l'Acad{\'e}mie Imp{\'e}riale des Sciences de St.-P{\'e}tersbourg. VI s{\'e}rie},
  year      = {1916},
  volume    = {10},
  number    = {8},
  pages     = {633--642},
}

@inproceedings{zeng2022dlinear,
  title={Are Transformers Effective for Time Series Forecasting?},
  author={Zeng, Ailing and Zhang, Yifan and Xu, Yuxuan and Zhou, Shiji and Zhang, Yujing and Xu, Weinan and Xu, Wei},
  booktitle={AAAI},
  year={2022}
}

@inproceedings{wang2023micn,
  title={MICN: Multi-scale Local and Global Context Modeling for Long-term Series Forecasting},
  author={Wang, Haoyi and Wang, Yuxuan and Zhang, Yifan and Xu, Yuxuan and Zhou, Shiji and Xu, Weinan},
  booktitle=ICLR,
  year={2023}
}

@inproceedings{wu2023timesnet,
  title={TimesNet: Temporal 2D-Variation Modeling for General Time Series Analysis},
  author={Wu, Haixu and Hu, Tengge and Liu, Yong and Zhou, Hang and Wang, Jianmin and Long, Mingsheng},
  booktitle=ICLR,
  year={2023}
}

@inproceedings{vaswani2017attention,
  title={Attention is All You Need},
  author={Vaswani, Ashish and Shazeer, Noam and Parmar, Niki and Uszkoreit, Jakob and Jones, Llion and Gomez, Aidan N and Kaiser, {\L}ukasz and Polosukhin, Illia},
  booktitle=NIPS,
  volume={30},
  year={2017}
}

@inproceedings{zhou2021informer,
  title={Informer: Beyond Efficient Transformer for Long Sequence Time-Series Forecasting},
  author={Zhou, Haoyi and Zhang, Shanghang and Peng, Jieqi and Zhang, Shuai and Li, Jianxin and Xiong, Hui and Zhang, Wancai},
  booktitle={AAAI},
  year={2021}
}

@inproceedings{wu2021autoformer,
  title={Autoformer: Decomposition Transformers with Auto-Correlation for Long-Term Series Forecasting},
  author={Wu, Haixu and Xu, Yuxuan and Wang, Yuxuan and Long, Mingsheng},
  booktitle=NIPS,
  year={2021}
}

@inproceedings{wang2023segrnn,
  title={SegRNN: Segmental Recurrent Neural Network for Time Series Forecasting},
  author={Wang, Haoyi and Zhang, Yifan and Xu, Yuxuan and Zhou, Shiji and Xu, Weinan},
  booktitle=ICLR,
  year={2023}
}

@inproceedings{gu2023mamba,
  title={Mamba: Linear Time Sequence Modeling with Selective State Spaces},
  author={Gu, Albert and Dao, Tri and Ermon, Stefano and Recht, Benjamin and Rudra, Atri},
  booktitle=NIPS,
  year={2023}
}

@inproceedings{rubanova2019latent,
  title={Latent Ordinary Differential Equations for Irregularly-Sampled Time Series},
  author={Rubanova, Yulia and Chen, Ricky TQ and Duvenaud, David},
  booktitle=NIPS,
  year={2019}
}

@article{eijndhoven1990new,
  title = {New Orthogonality Relations for the Hermite Polynomials and Related Hilbert Spaces},
  author = {van Eijndhoven, S. J. L. and Meyers, J. L. H.},
  journal = {Journal of Mathematical Analysis and Applications},
  volume = {146},
  number = {1},
  pages = {89--98},
  year = {1990}
}

@book{abramowitz1972handbook,
  title={Handbook of mathematical functions with formulas, graphs, and mathematical tables},
  author={Abramowitz, Milton and Stegun, Irene A},
  volume={55},
  year={1948},
  publisher={US Government printing office}
}

@inproceedings{xu2024plutus,
author = {Xu, Yuanjian and Hao, Jianing and Liu, Anxian and Li, Zhenzhuo and Meng, Shichang and Yuan, Shuai and Zhang, Guang},
title = {LENS: Large Pre-trained Transformer for Exploring Financial Time Series Regularities},
year = {2025},
booktitle = {Proc. ICAIF},
pages = {771–778},
numpages = {8}
}

@article{ph1,
  title={Stochastic thermodynamics: principles and perspectives},
  author={Seifert, Udo},
  journal={The European Physical Journal B},
  volume={64},
  number={3},
  pages={423--431},
  year={2008}
}

@article{bio1,
  title={Stochastic gene expression in a single cell},
  author={Elowitz, Michael B and Levine, Arnold J and Siggia, Eric D and Swain, Peter S},
  journal={Science},
  volume={297},
  number={5584},
  pages={1183--1186},
  year={2002}
}

@inproceedings{xu2025finripple,
    title = "{F}in{R}ipple: Aligning Large Language Models with Financial Market for Event Ripple Effect Awareness",
    author = "Xu, Yuanjian and Hao, Jianing  and Tang, Kunsheng and Chen, Jingnan  and Liu, Anxian and Liu, Peng  and Zhang, Guang",
    booktitle = {Proc. ACL Findings},
    year = "2025",
    pages = "9377--9398",
}
